\documentclass[12pt]{article}
\usepackage[centertags]{amsmath}
\usepackage{amsfonts}
\usepackage{amssymb}
\usepackage{latexsym}
\usepackage{amsthm}
\usepackage{newlfont}
\usepackage{graphicx}
\usepackage{listings}
\usepackage{booktabs}
\usepackage{abstract}
\newlength{\defbaselineskip}
\newcommand{\setlinespacing}[1]%
           {\setlength{\baselineskip}{#1 \defbaselineskip}}

\newcommand{\actaqed}{\hfill $\actabox$}
{\medskip\noindent \textit{Proof of #1. }}%
{\actaqed \medskip}

\def\D{{\mathcal D}}

\def\R{{\mathbb R}}
\def \<{\langle}
\def\>{\rangle}

\def \e{\epsilon}

\def \ff{\varphi}

\def \sp{\operatorname{span}}

\def\la{\lambda}

\newtheorem{Theorem}{Theorem}[section]
\newtheorem{Lemma}{Lemma}[section]

\numberwithin{equation}{section}

\begin{document}
\title{{Chebushev Greedy Algorithm in convex optimization} }
\author{V.N. Temlyakov \thanks{ University of South Carolina and Steklov Institute of Mathematics. Research was supported by NSF grant DMS-1160841 }} \maketitle
\begin{abstract}
{Chebyshev Greedy Algorithm is a generalization of the well known Orthogonal Matching Pursuit defined in a Hilbert space to the case of Banach spaces. We apply this algorithm for  constructing sparse approximate solutions (with respect to a given dictionary) to convex optimization problems. Rate of convergence results in a style of the Lebesgue-type inequalities are proved.   }
\end{abstract}

\section{Introduction}

We study sparse approximate solutions to convex optimization problems. We apply the technique developed in nonlinear approximation known under the name of {\it greedy approximation}. A typical 
problem of convex optimization is to find an approximate solution to the problem
\begin{equation}\label{1.0}
\inf_x E(x)
\end{equation}
under assumption that $E$ is a convex function. Usually, in convex optimization function $E$ is defined on a finite dimensional space $\R^n$ (see \cite{BL}, \cite{N}).
Recent needs of numerical analysis call for consideration of the above optimization problem on an infinite dimensional space, for instance, a space of 
continuous functions.  
Thus, we consider a convex function $E$ defined on a Banach space $X$. This paper is a follow up to papers \cite{Tco1}, \cite{Tco2}, and \cite{SSZ}. We refer the reader to the above mentioned papers for a detailed discussion and justification of importance of greedy methods in optimization problems. 

 Let $X$ be a Banach space with norm $\|\cdot\|$. We say that a set of elements (functions) $\D$ from $X$ is a dictionary, respectively, symmetric dictionary, if each $g\in \D$ has norm bounded by one ($\|g\|\le1$),
$$
g\in \D \quad \text{implies} \quad -g \in \D,
$$
and the closure of $\sp \D$ is $X$. For notational convenience in this paper symmetric dictionaries are considered. Results of the paper also hold for non-symmetric dictionaries with straight forward modifications. We denote the closure (in $X$) of the convex hull of $\D$ by $A_1(\D)$. In other words $A_1(\D)$ is the closure of conv($\D$). We use this notation because it has become a standard notation in relevant greedy approximation literature.
 
 We 
assume that $E$ is Fr{\'e}chet differentiable and that the set
$$
D:=\{x:E(x)\le E(0)\}
$$
is bounded.
For a bounded set $D$ define the modulus of smoothness of $E$ on $D$ as follows
\begin{equation}\label{1.1}
\rho(E,u):=\frac{1}{2}\sup_{x\in D, \|y\|=1}|E(x+uy)+E(x-uy)-2E(x)|.
\end{equation}
We say that $E$ is {\it uniformly smooth} if $\rho(E,u)=o(u)$, $u\to 0$. 

We defined and studied in \cite{Tco1} the following generalization of the Weak Chebyshev Greedy Algorithm (see \cite{Tbook}, Ch. 6) for convex optimization.

 {\bf Weak Chebyshev Greedy Algorithm (WCGA(co)).} Let $\tau :=\{t_k\}_{k=1}^\infty$, $t_k\in (0,1]$, $k=1,2,\dots$, be a weakness sequence. 
We define $G_0 := 0$. Then for each $m\ge 1$ we have the following inductive definition.

(1) $\varphi_m :=\varphi^{c,\tau}_m \in \D$ is any element satisfying
$$
\<-E'(G_{m-1}),\varphi_m\> \ge t_m  \sup_{g\in \D}\< -E'(G_{m-1}),g\>.
$$

(2) Define
$$
\Phi_m := \Phi^\tau_m := \sp \{\varphi_j\}_{j=1}^m,
$$
and define $G_m := G_m^{c,\tau}$ to be the point from $\Phi_m$ at which $E$ attains the minimum:
$$
E(G_m)=\inf_{x\in \Phi_m}E(x).
$$ 

We consider here along with the WCGA(co) the following greedy algorithm.

{\bf $E$-Greedy Chebyshev Algorithm (EGCA(co)).} 
We define $G_0 := 0$. Then for each $m\ge 1$ we have the following inductive definition.

(1) $\varphi_m :=\varphi^{E,\tau}_m \in \D$ is any element satisfying (assume existence)
$$
\inf_{c}E(G_{m-1}+c\varphi_m) = \inf_{c,g\in\D}E(G_{m-1}+cg).
$$

(2) Define
$$
\Phi_m := \Phi^\tau_m := \sp \{\varphi_j\}_{j=1}^m,
$$
and define $G_m := G_m^{E,\tau}$ to be the point from $\Phi_m$ at which $E$ attains the minimum:
$$
E(G_m)=\inf_{x\in \Phi_m}E(x).
$$ 

The EGCA(co) is in a style of $X$-Greedy algorithms studied in approximation theory (see \cite{Tbook}, Ch. 6). In a special case of $X=\R^d$ and $\D$ is a canonical basis of $\R^d$ the EGCA(co) was introduced and studied in \cite{SSZ}. Convergence and rate of convergence of the WCGA(co) were studied in \cite{Tco1}. For instance, the following rate of convergence theorem was proved in \cite{Tco1}. 

\begin{Theorem}\label{T1.1} Let $E$ be a uniformly smooth convex function with modulus of smoothness $\rho(E,u)\le \gamma u^q$, $1<q\le 2$. Take a number $\e\ge 0$ and an element  $f^\e$ from $D$ such that
$$
E(f^\e) \le \inf_{x\in D}E(x)+ \e,\quad
f^\e/B \in A_1(\D),
$$
with some number $B\ge 1$.
Then we have for the WCGA(co) ($p:=q/(q-1)$)
\begin{equation}\label{2.11o}
E(G_m)-\inf_{x\in D}E(x) \le  \max\left(2\e, C(q,\gamma)B^q\left(C(E,q,\gamma)+\sum_{k=1}^mt_k^p\right)^{1-q}\right) . 
\end{equation}
\end{Theorem}

We will use the following notations. Let $f_0$ be a point of minimum of $E$:
$$
E(f_0) = \inf_{x\in D}E(x).
$$
We denote for $m=1,2,\dots$
$$
f_m:=f_0-G_m.
$$

In particular, if the point of minimum $f_0$ belongs to $A_1(\D)$, then Theorem \ref{T1.1} in the case $t_k=t\in(0,1)$, $k=1,\dots$, with $\e=0$, $B=1$, gives
\begin{equation}\label{1.4}
E(G_m)-E(f_0)\le C(q,\gamma,t)m^{1-q}.
\end{equation}
Inequality (\ref{1.4}) uses only information that $f_0\in A_1(\D)$. 
Theorem \ref{T1.1} is designed in a way that the convergence rate is determined by smoothness of $E$ and complexity of $f_0$.   Our way of measuring complexity of the element $f_0$ in Theorem \ref{T1.1} is based on $A_1(\D)$. Given a dictionary $\D$ we say that $f_0$ is {\it simple} with respect to $\D$ if $f_0\in A_1(\D)$. 
Next, let for every $\e>0$ an element $f^\e$ be such that 
$$
E(f^\e)\le E(f_0)+ \e,\qquad f^\e/A(\e) \in A_1(\D)
$$
with some number $A(\e)$ (the smaller the $A(\e)$ the better). Then we say that complexity of $f_0$ is bounded (bounded from above) by the function $A(\e)$. 

 We apply algorithms which at the $m$th iteration provide an $m$-term polynomial $G_m$ with respect to $\D$. The approximant belongs to the domain $D$ of our interest. Then on one hand we always have the lower bound
$$
E(G_m) - \inf_{x\in D} E(x) \ge \inf_{x\in D\cap \Sigma_m(\D)} E(x) - \inf_{x\in D} E(x)
$$
where $\Sigma_m(\D)$ is a collection of all $m$-term polynomials with respect to $\D$.
On the other hand if we know $f_0$ then the best we can do with our algorithms is to get
$$
\|f_0-G_m\| = \sigma_m(f_0,\D)
$$
where $\sigma_m(f_0,\D)$ is the best $m$-term approximation of $f_0$ with respect to $\D$.
Then we can aim at building algorithms that provide an error $E(G_m)-E(f_0)$ comparable to $\rho(E,\sigma_m(f_0,\D))$. It would be in a style of the Lebesgue-type inequalities. However, it is known from greedy approximation theory that there is no Lebesgue-type inequalities which hold for an arbitrary dictionary even in the case of Hilbert spaces. There are the Lebesgue-type inequalities for special dictionaries.  We refer the reader to \cite{Tbook}, \cite{LivTem}, \cite{T144}, \cite{Z2} for results on the Lebesgue-type inequalities. In this paper we obtain rate of convergence results for the WCGA(co) in a style of the Lebesgue-type inequalities.

We will use the following assumptions on properties of $E$.

{\bf E1. Smoothness.} We assume that $E$ is a convex function with 
$$\rho(E,u)\le \gamma u^2.$$

{\bf E2. Restricted strong convexity.} We assume that for any $S$-sparse element $f$ we have
\begin{equation}\label{E2}
E(f)-E(f_0) \ge \beta \|f-f_0\|^2.
\end{equation}

Here is one assumption on the dictionary $\D$ that we will use (see \cite{T144}).  For notational simplicity we formulate it for a countable dictionary $\D=\{g_i\}_{i=1}^\infty$. 
 
{\bf A.} We say that $f=\sum_{i\in T}x_ig_i$ has $\ell_1$ incoherence property with parameters $S$, $V$, and $r$ if for any $A\subset T$ and any $\Lambda$ such that $A\cap \Lambda =\emptyset$, $|A|+|\Lambda| \le S$ we have for any $\{c_i\}$
\begin{equation}\label{C3}
\sum_{i\in A}|x_i| \le V|A|^r\|f_A-\sum_{i\in\Lambda}c_ig_i\|,\quad f_A:=\sum_{i\in A} x_ig_i.
\end{equation}
A dictionary $\D$ has $\ell_1$ incoherence property with parameters $K$, $S$, $V$, and $r$ if for any $A\subset B$, $|A|\le K$, $|B|\le S$ we have for any $\{c_i\}_{i\in B}$
$$
\sum_{i\in A} |c_i| \le V|A|^r\|\sum_{i\in B} c_ig_i\|.
$$

 The following theorem is the main result of the paper. 
\begin{Theorem}\label{T2.1} Let $E$ satisfy assumptions {\bf E1} and {\bf E2}. Suppose for a point of minimum $f_0$ we have $\|f_0-f^\e\|\le \e$ with $K$-sparse $f:=f^\e$ satisfying property {\bf A}. Then for the WCGA(co) with weakness parameter $t$ we have for $K+m \le S$
$$
E(G_m)-E(f_0) \le \max\left((E(0)-E(f_0))\exp\left(-\frac{c_1m}{K^{2r}}\right), 8(\gamma^2/\beta)\e^2\right)+2\gamma \e^2, 
$$
where $c_1:= \frac{\beta t^2}{64\gamma V^2}$.
\end{Theorem}
 
 Let us apply Theorem \ref{T2.1} in a particular case $r=1/2$. If we assume that $\sigma_K(f_0,\D) \le C_1K^{-s}$ then for $m$ of order $K\ln K$ Theorem \ref{T2.1} with $\e=C_1K^{-s}$ provides the bound
 $$
 E(G_m)-E(f_0) \le C_2 K^{-2s}.
 $$
 Note that $K^{-2s}$ is of oder $\rho(E,K^{-s})$ in our case. 
 
 In the case of direct application of the Weak Chebyshev Greedy Algorithm to the element $f_0$ the corresponding results in a style of the Lebesgue-type inequalities are known (see \cite{LivTem} and \cite{T144}).

\section{Proofs}

We assume that $E$ is Fr{\'e}chet differentiable. Then convexity of $E$ implies that for any $x,y$ 
\begin{equation}\label{1.2}
E(y)\ge E(x)+\<E'(x),y-x\>
\end{equation}
or, in other words,
\begin{equation}\label{1.3}
E(x)-E(y) \le \<E'(x),x-y\> = \<-E'(x),y-x\>.
\end{equation} 
We will often use the following simple lemma (see \cite{Tco1}).
\begin{Lemma}\label{L1.1} Let $E$ be Fr{\'e}chet differentiable convex function. Then the following inequality holds for $x\in D$
\begin{equation}\label{1.6}
0\le E(x+uy)-E(x)-u\<E'(x),y\>\le 2\rho(E,u\|y\|).  
\end{equation}
\end{Lemma} 

  The following two simple lemmas are  well-known (see \cite{Tbook}, Chapter 6 and \cite{Tco1}, Section 2).
\begin{Lemma}\label{L2.1} Let $E$ be a uniformly smooth convex function on a Banach space $X$ and $L$ be a finite-dimensional subspace of $X$. Let $x_L$ denote the point from $L$ at which $E$ attains the minimum:
$$
E(x_L)=\inf_{x\in L}E(x).
$$ 
 Then we have 
$$
\<E'(x_L),\phi\> =0
$$
for any $\phi \in L$.
\end{Lemma}
 
\begin{Lemma}\label{L2.2} For any bounded linear functional $F$ and any dictionary $\D$, we have
$$
 \sup_{g\in \D}\<F,g\> = \sup_{f\in  A_1(\D)} \<F,f\>.
$$
\end{Lemma}
{\it Proof of Theorem \ref{T2.1}.}  Let 
$$
 f:=f^\e=\sum_{i\in T}x_ig_i,\quad g_i\in\D,\quad |T|=K.
 $$
 We examine $n$ iterations of the algorithm for $n=1,\dots,m$.  Denote by $T^n$ the set of indices of $g_i$ picked by the WCGA(co) after $n$ iterations, $\Gamma^n := T\setminus T^n$.
 Denote as above by $A_1(\D)$ the closure in $X$ of the convex hull of the symmetric dictionary $\D$.   We will bound from above $a_n:=E(G_n)-E(f^\e)$.  Assume $\|f_{n-1}\|^2\ge 4(\gamma/\beta) \e^2$ for all $n=1,\dots,m$.   Denote $A_n:=\Gamma^{n-1}$ and 
 $$
 f_{A_n}:= f^\e_{A_n} := \sum_{i\in A_n} x_ig_i, \quad \|f_{A_n}\|_1:=\sum_{i\in A_n} |x_i|.    
 $$    
 The following lemma is used in our proof.
\begin{Lemma}\label{L2.3} Let $E$ be a uniformly smooth convex function   with modulus of smoothness $\rho(E,u)$. Take a number $\e\ge 0$ and a $K$-sparse element  $f^\e=\sum_{i\in T} x_ig_i$ from $D$ such that
$$
\|f_0-f^\e\|\le \e.
$$
Then we have for the WCGA(co)
$$
E(G_n)-E(f^\e) \le E(G_{n-1})-E(f^\e) 
$$
$$
+ \inf_{\la\ge0}(-\la t\|f_{A_n}\|_1^{-1}(E(G_{n-1})-E(f^\e)) + 2\rho(E,\la)),
$$
for $ n=1,2,\dots$ .
\end{Lemma}
\begin{proof} It follows from the definition of WCGA(co) that $E(0)\ge E(G_1)\ge E(G_2)\dots$. Therefore, if  $E(G_{n-1})-E(f^\e)\le 0$ then the claim of Lemma \ref{L2.3} is trivial. Assume $E(G_{n-1})-E(f^\e)>0$. By Lemma \ref{L1.1} we have for any $\la$
\begin{equation}\label{2.3}
E(G_{n-1}+\la \varphi_n) \le E(G_{n-1}) - \la\<-E'(G_{n-1}),\varphi_n\> + 2 \rho(E,\la)
\end{equation}
and by (1) from the definition of the WCGA(co) and Lemma \ref{L2.2} we get
$$
\<-E'(G_{n-1}),\varphi_n\> \ge t \sup_{g\in \D} \<-E'(G_{n-1}),g\> =
$$
$$
t\sup_{\phi\in A_1(\D)} \<-E'(G_{n-1}),\phi\> \ge t \|f_{A_n}\|_1^{-1} \<-E'(G_{n-1}),f_{A_n}\>.
$$
By Lemma \ref{L2.1} and (\ref{1.3}) we obtain
$$
\<-E'(G_{n-1}),f_{A_n}\> = \<-E'(G_{n-1}),f^\e-G_{n-1}\> \ge E(G_{n-1})-E(f^\e).
$$
 Thus,  
$$
E(G_n) \le \inf_{\la\ge0} E(G_{n-1}+ \la\varphi_n)  
$$
\begin{equation}\label{2.4}
\le E(G_{n-1}) + \inf_{\la\ge0}(-\la t\|f_{A_n}\|_1^{-1}(E(G_{n-1})-E(f^\e)) + 2\rho(E,\la)),  
\end{equation}
which proves the lemma.
\end{proof}
Denote
$$
a_n:=E(G_n)-E(f^\e).
$$
From (\ref{2.4})  we obtain
 \begin{equation}\label{2.5}
 a_n \le a_{n-1} +\inf_{\la\ge 0} \left(-\la t \frac{a_{n-1}}{\|f_{A_n}\|_1} + 2\rho(E,\la)\right).
 \end{equation}
 By assumption {\bf E1} we have $\rho(E,u) \le \gamma u^2$. We get from (\ref{2.5})
 $$
 a_n \le a_{n-1}+\inf_{\la\ge 0}\left(-\frac{\la t a_{n-1}}{\|f_{A_n}\|_1} +2\gamma \la^2 \right)  .
 $$
 Let $\la_1$ be a solution of 
 $$
\frac{\la t a_{n-1}}{2 \|f_{A_n}\|_1} = 2\gamma  \la^2 ,\quad \la_1 = \frac{ t a_{n-1}}{4\gamma  \|f_{A_n}\|_1}.
 $$
 Our assumption {\bf A} (see (\ref{C3})) gives
 $$
 \|f_{A_n}\|_1=\|(f^\e-G_{n-1})_{A_n}\|_1 \le VK^r\|f^\e-G_{n-1}\| 
 $$
 \begin{equation}\label{2.6}
 \le VK^r(\|f_0-G_{n-1}\|+\|f_0-f^\e\|) \le VK^r(\|f_{n-1}\|+\e).
 \end{equation}
 
 We bound from below $a_{n-1}=E(G_{n-1})-E(f^\e)$. By our smoothness assumption and Lemma \ref{L1.1}
 $$
 E(f^\e)-E(f_0) \le 2\gamma \|f^\e-f_0\|^2 \le 2\gamma \e^2.
 $$
 Therefore,
 $$
 a_{n-1} = E(G_{n-1})-E(f^\e) = E(G_{n-1})-E(f_0) +E(f_0)-E(f^\e) 
 $$
 $$
  \ge E(G_{n-1})-E(f_0) -2\gamma \e^2.
  $$
 By restricted strong convexity assumption {\bf E2}
 $$
 E(G_{n-1})-E(f_0) \ge \beta \|G_{n-1}-f_0\|^2 = \beta \|f_{n-1}\|^2.
 $$
 Thus
  \begin{equation}\label{2.7}
 a_{n-1}\ge  \beta \|f_{n-1}\|^2-2\gamma\e^2.
 \end{equation}
 Specify
 $$
 \la = \frac{t \beta \|f_{A_n}\|_1}{32\gamma (VK^r)^2}.
 $$
 Then, using (\ref{2.6}) and (\ref{2.7}) we get
 \begin{equation}\label{2.8}
 \frac{\la}{\la_1} = \frac{\beta \|f_{A_n}\|_1^2}{8(VK^r)^2a_{n-1}}\le \frac{\beta(\|f_{n-1}\|+\e)^2}{8(\beta\|f_{n-1}\|^2-2\gamma \e^2)}.
 \end{equation}
   By our assumption $\|f_{n-1}\|^2 \ge 4(\gamma/\beta)\e^2$ and a trivial inequality $\beta \le 2\gamma$ we obtain from (\ref{2.8}) that $\la \le \la_1$ and therefore 
 $$
 a_n \le a_{n-1} \left( 1- \frac{\beta t^2}{64\gamma  (VK^r)^2 }\right),\quad n=1,\dots,m .
 $$
 Denote $c_1:= \frac{\beta t^2}{64\gamma V^2}$. Then
 \begin{equation}\label{2.9}
 a_m \le a_0\exp\left(-\frac{c_1m}{K^{2r}}\right) .
 \end{equation}
 We obtained (\ref{2.9}) under assumption $\|f_{n-1}\|^2 \ge 4(\gamma/\beta)\e^2$, $n=1,\dots,m$.
 If $\|f_{n-1}\|^2 < 4(\gamma/\beta)\e^2$ for some $n\in [1,m]$ then $a_{m-1}\le a_{n-1}\le 2\gamma \|f_{n-1}\|^2 \le 8(\gamma^2/\beta)\e^2$.
 Therefore,
 $$
 a_m \le \max\left(a_0\exp\left(-\frac{c_1m}{K^{2r}}\right), 8(\gamma^2/\beta)\e^2\right).
 $$
 Next, we have 
 $$
 E(G_m)-E(f_0) = a_m+ E(f^\e)-E(f_0) \le a_m+2\gamma \e^2.
 $$
 This completes  the proof of Theorem \ref{T2.1}. 

The above technique of studying the WCGA(co) works for the EGCA(co) as well. Instead of Lemma \ref{L2.3} we have the following one.
\begin{Lemma}\label{L2.4} Let $E$ be a uniformly smooth convex function   with modulus of smoothness $\rho(E,u)$. Take a number $\e\ge 0$ and a $K$-sparse element  $f^\e$ from $D$ such that
$$
\|f_0-f^\e\|\le \e.
$$
Then we have for the EGCA(co)
$$
E(G_n)-E(f^\e) \le E(G_{n-1})-E(f^\e) 
$$
$$
+ \inf_{\la\ge0}(-\la \|f_{A_n}\|_1^{-1}(E(G_{n-1})-E(f^\e)) + 2\rho(E,\la)),
$$
for $ n=1,2,\dots$ .
\end{Lemma}
\begin{proof} In the proof of Lemma \ref{L2.3} we did not use a specific form of the $G_{n-1}$ as the one generated by the $(n-1)$th iteration of the WCGA(co), we only used that $G_{n-1}\in D$. Let $G_{n-1}$ be from the $(n-1)$th iteration of the EGCA(co) and let $\ff_m^t$, $t\in(0,1)$, be such that 
$$
\<-E'(G_{n-1}),\ff_m^t\> \ge t\sup_{g\in\D}\<-E'(G_{n-1}),g\>.
$$
Then the above proof of Lemma \ref{L2.3} gives
\begin{equation}\label{2.10}
\inf_{\la\ge 0} E(G_{n-1}+\la \ff_m^t) \le \inf_{\la\ge0}(-\la t\|f_{A_n}\|_1^{-1}(E(G_{n-1})-E(f^\e)) + 2\rho(E,\la)).
\end{equation}
Definition of the EGCA(co) implies
 \begin{equation}\label{2.11}
E(G_m)\le \inf_{c} E(G_{n-1}+c \ff_m)\le\inf_{\la\ge 0} E(G_{n-1}+\la \ff_m^t).
\end{equation}
Combining (\ref{2.10}) and (\ref{2.11}) and taking into account that $E(G_m)$ does not depend on $t$, we complete the proof of Lemma \ref{L2.4}. 
 
\end{proof}

The following theorem is derived from Lemma \ref{L2.4} in the same way as Theorem \ref{T2.1} was derived from Lemma \ref{L2.3}. 

\begin{Theorem}\label{T2.2} Let $E$ satisfy assumptions {\bf E1} and {\bf E2}. Suppose for a point of minimum $f_0$ we have $\|f_0-f^\e\|\le \e$ with $K$-sparse $f:=f^\e$ satisfying property {\bf A}. Then for the EGCA(co)  we have for $K+m \le S$
$$
E(G_m)-E(f_0) \le \max\left((E(0)-E(f_0))\exp\left(-\frac{c_1m}{K^{2r}}\right), 8(\gamma^2/\beta)\e^2\right)+2\gamma \e^2, 
$$
where $c_1:= \frac{\beta }{64\gamma V^2}$.
\end{Theorem}

\end{document}